\documentclass[wide14,final,a4paper,11pt,oneside,bind2cm0,
               onecolumn,english,titlepage,tocpage]{LIArep}
\input{LIArepTm.sty}
\usepackage{ifthen}
\newboolean{long}
\setboolean{long}{false}
\usepackage[english]{babel}
\usepackage[numbers,sort&compress]{natbib}

\usepackage{graphicx}
\usepackage{url}
\usepackage{array}
\usepackage{subfig}
\usepackage{listings}
\usepackage[latin1]{inputenc}
\usepackage{fancybox}
\usepackage{enumerate}
\usepackage{amsmath}
\usepackage{amssymb}
\usepackage{amsthm}
\usepackage{algorithm,algorithmic}
\usepackage[usenames]{color}
\usepackage{gastex}

\theoremstyle{plain}
\newtheorem{proposition}{Proposition}
\theoremstyle{definition}
\newtheorem{definition}{Definition}
\theoremstyle{remark}
\newtheorem{note}{Note}

\newcommand{\resolve}{\triangleright}
\newcommand{\resolved}{\triangleleft}

\newcommand{\hresolve}{\trianglelefteq}
\newcommand{\hresolved}{\trianglerighteq}

\newcommand{\compatible}{\mapsto}
\newcommand{\hpp}[0]{Haplotype Inference}

\begin{document}


\title{A preliminary analysis on metaheuristics methods applied to the \hpp{} Problem}

\author{
    Luca Di Gaspero\postref{1}
    \and
    Andrea Roli\postref{2}
}

\address{
   \prelabel{1}DIEGM, University of Udine \\
  	via delle Scienze 208, I-33100, Udine, Italy \\
  	\texttt{l.digaspero@uniud.it} 
   \and
   \prelabel{2}DEIS, Campus of Cesena\\
	University of Bologna\\
	 via Venezia 52, I-47023 Cesena, Italy \\
  \texttt{andrea.roli@unibo.it}
}

\date{August 3, 2007}

\deistechno{DEIS-LIA-006-07}
													   
\liatechno{84}	
														   

\abstractext{%
\hpp{} is a challenging problem in bioinformatics that consists in inferring the basic genetic constitution of diploid organisms on the basis of their genotype. This information allows researchers to perform association studies for the genetic variants involved in diseases and the individual responses to therapeutic agents.

A notable approach to the problem is to encode it as a combinatorial problem (under certain hypotheses, such as the \emph{pure parsimony} criterion) and to solve it using off-the-shelf combinatorial optimization techniques. The main methods applied to \hpp{} are either simple greedy heuristic or exact methods (Integer Linear Programming, Semidefinite Programming, SAT encoding) that, at present, are adequate only for moderate size instances. 

We believe that metaheuristic and hybrid approaches could provide a better scalability. Moreover, metaheuristics can be very easily combined with problem specific heuristics and they can also be integrated with tree-based search techniques, thus providing a promising framework for hybrid systems in which a good trade-off between effectiveness and efficiency can be reached. 

In this paper we illustrate a feasibility study of the approach and discuss some relevant design issues, such as modeling and design of approximate solvers that combine constructive heuristics, local search-based improvement strategies and learning mechanisms. Besides the relevance of the \hpp{} problem itself, this preliminary analysis is also an interesting case study because the formulation of the problem poses some challenges in modeling and hybrid metaheuristic solver design that can be generalized to other problems.

}

\maketitle

%

\section{Introduction}

A fundamental tool of analysis to investigate the genetic variations in a population is based on \emph{haplotype} data. A haplotype is a copy of a chromosome of a diploid organism (i.e., an organism that has two copies of each chromosome, one inherited from the father and one from the mother).
The haplotype information allows to perform association studies for the genetic variants involved in diseases and the individual responses to therapeutic agents. The assessment of a full Haplotype Map of the human genome is indeed one of the current high priority tasks of human genomics \cite{HapMap03}.

Instead of dealing with complete DNA sequences, usually the researchers are focusing on \emph{Single Nucleotide Polymorphisms} (SNPs), which are the most common mutations among haplotypes. A SNP is a single nucleotide site (allele) where exactly two (out of four) different nucleotides occur in a large percentage of the population.

The haplotype collection is not an easy task: in fact, due to technological limitations it is currently infeasible to directly collect haplotypes in an experimental way, but rather it is possible to collect \emph{genotypes}, i.e., the conflation of a pair of haplotypes. Moreover, instruments can only identify whether the individual is \emph{homozygous} (i.e., the alleles are the same) or \emph{heterozygous} (i.e., the alleles are different) at a given site.
Therefore, haplotypes have to be inferred from genotypes in order to reconstruct the detailed information and trace the precise structure of human populations. This process is called \emph{\hpp{}} and the goal is to find a set of haplotype pairs so that all the genotypes are \emph{resolved}.

The main approaches to solve the \hpp{} are either combinatorial or statistical methods. However, both of them, being of non-experimental nature, need some genetic model of haplotype evolution, which poses some hypotheses to constrain the possible inferences to the ones compatible with Nature. In the case of the combinatorial methods, which are the subject of the present work, a reasonable criterion is the \emph{pure parsimony} approach \cite{DBLP:conf/cpm/Gusfield03}, which searches for the smallest collection of distinct haplotypes that solves the \hpp{} problem. This criterion is consistent with current observations in natural populations for which the actual number of haplotypes is vastly smaller than the total number of possible haplotypes. 


Current approaches for solving the problem include simple greedy heuristic \cite{Clark90} and exact methods such as Integer Linear Programming \cite{DBLP:conf/cpm/Gusfield03,DBLP:conf/recomb/HalldorssonBELYI02,DBLP:journals/informs/LanciaPR04,DBLP:journals/tcbb/BrownH06}, Semidefinite Programming \cite{DBLP:conf/bibe/KalpakisN05,DBLP:conf/sac/HuangCC05} and SAT models \cite{DBLP:conf/sat/LynceM06,DBLP:conf/aaai/LynceM06}.
These approaches, however, at present are adequate only for moderate size instances.

To the best of our knowledge, the only attempt to employ metaheuristic techniques for the problem is a recently proposed Genetic Algorithm \cite{WaZS05}. However, also the cited paper does not report results on real size instances. Anyway, we believe that metaheuristic approaches could be effective to this problem since they could provide a better scalability.

In this work we make a preliminary conceptual analysis on the use of metaheuristics for the \hpp{} problem. We start introducing the \hpp{} problem in Section~\ref{sec:hpp-definition} and then we present two possible local search models for the problem (Section~\ref{sec:local-search}) highlighting the possible benefits and drawbacks of each model. Section~\ref{sec:metaheuristic} contains the description of metaheuristic approaches that, in our opinion, could be adequate for \hpp{}. In Section~\ref{sec:constructive} we consider the role of constructive techniques in the hybridization with metaheuristics and, finally, in Section~\ref{sec:discussion} we discuss our proposals and outline future developments.

\section{The \hpp{} problem}
\label{sec:hpp-definition}


In the \hpp{} problem we deal with \emph{genotypes}, that is, strings of length $m$ that corresponds to a chromosome with $m$ sites. Each value in the string belongs to the alphabet $\{0, 1, 2\}$. A position in the genotype is associated with a site of interest on the chromosome (e.g., a SNP) and it has value 0 (wild type) or 1 (mutant) if the corresponding chromosome site is a homozygous site (i.e., it has that state on both copies) or the value 2 if the chromosome site is heterozygous.
A \emph{haplotype} is a string of length $m$ that corresponds to only one copy of the chromosome (in diploid organisms) and whose positions can assume the symbols $0$ or $1$.

\subsection{Genotype resolution}

Given a chromosome, we are interested in finding an unordered\footnote{In the problem there is no distinction between the maternal and paternal haplotypes.} pair of haplotypes that can explain the chromosome according to the following definition:

\begin{definition}[Genotype resolution]
Given a chromosome $g$, we say that the pair $\langle h, k \rangle$ \emph{resolves} $g$, and we write $\langle h, k \rangle \resolve g$ (or $g = h \oplus k$), if the following conditions hold (for $j = 1, \ldots, m$):
\begin{subequations}
\begin{eqnarray}
g[j] = 0 & \Rightarrow & h[j] = 0 \wedge k[j] = 0 \label{eq:chromosome0} \\
g[j] = 1 & \Rightarrow & h[j] = 1 \wedge k[j] = 1 \label{eq:chromosome1} \\
g[j] = 2 & \Rightarrow & (h[j] = 0  \wedge k[j] = 1) \vee (h[j] = 1  \wedge k[j] = 0) \label{eq:chromosome2}
\end{eqnarray}
\end{subequations}

If $\langle h, k \rangle \resolve g$ we indicate the fact that the haplotype $h$ (respectively, $k$) contributes in the resolution of the genotype $g$ writing $h \hresolve g$ (resp., $k \hresolve g$). We extend this notation to set of haplotypes and we write $H = \{h_1, \ldots, h_l\} \hresolve g$, meaning that $h_i \hresolve g$ for all $i = 1, \ldots, l$.
\end{definition}

Conditions~(\ref{eq:chromosome0}) and (\ref{eq:chromosome1}) require that both haplotypes must have the same value in all homozygous sites, while condition~(\ref{eq:chromosome2}) states that in heterozygous sites the haplotypes must have different values.

Observe that, according to the definition, for a single genotype string the haplotype values at a given site are predetermined in the case of homozygous sites, whereas there is a freedom to choose between two possibilities at heterozygous places. This means that for a genotype string with $l$ heterozygous sites there are $2^{l-1}$ possible pairs of haplotypes that resolve it.

As an example, consider the genotype $g = (0 2 1 2)$, then the possible pairs of haplotypes that resolve it are $\langle (0 1 1 0), (0 0 1 1) \rangle$ and $\langle (0 0 1 0), (0 1 1 1) \rangle$.

After these preliminaries we can state the \emph{\hpp{}} problem as follows: 

\begin{definition}[\hpp{} problem]
Given a population of $n$ individuals, each of them represented by a genotype string $g_i$ of length $m$ we are interested in finding a set $R$ of $n$ pairs of (not necessarily distinct) haplotypes $R = \{ \langle h_1, k_1 \rangle, \ldots, \langle h_n, k_n \rangle\}$, so that $\langle h_i, k_i \rangle \resolve g_i, i = 1,\ldots,n$. We call $H$ the set of haplotypes used in the construction of $R$, i.e., $H = \{h_1, \ldots, h_n, k_1, \ldots, k_n\}$.
\end{definition}

Of course, from the mathematical point of view, there are many possibility in constructing the set $R$ since there is an exponential number of possible haplotypes for each genotype.
However, the knowledge about the biological phenomenons should guide the selection of the haplotypes used by constraining it through a genetic model.

One natural model to the \hpp{} problem is the already mentioned \emph{pure parsimony} approach that consists in searching for a solution that minimizes the total number of distinct haplotypes used or, in other words, $|H|$, the cardinality of the set $H$. A trivial upper bound for $|H|$ is $2n$ in the case of all genotypes resolved by a pair of distinct haplotypes.

It has been shown that the \hpp{} problem under the pure parsimony hypothesis is APX-hard \cite{DBLP:journals/informs/LanciaPR04} and therefore NP-hard.

\subsection{Compatibility and complementarity}

It is possible to define a graph that express the compatibility between genotypes, so as to avoid unnecessary checks in the determination of the resolvents.\footnote{In some cases, also a graph representing incompatibilities between genotypes can provide useful information.} Let us build the graph $\mathcal{G} = (G, E)$, in which the set of vertices coincides with the set of the genotypes; in the graph, a pair of genotypes $g_1, g_2$ are connected by an edge whether they are \emph{compatible}, i.e., one or more common haplotypes can resolve both of them. For example, the genotypes $(2210)$ and $(1220)$ are compatible, whereas genotypes $(2210)$ and $(1102)$ are not compatible. The formal definition of this property is as follows.

\begin{definition}[Genotypes compatibility]
Let $g_1$ and $g_2$ be two genotypes, $g_1$ and $g_2$ are \emph{compatible} if, for all $j=1, \ldots, m$, the following conditions hold:
\begin{subequations}
\begin{eqnarray}
g_1[j] = 0 & \Rightarrow & g_2[j] \in \{0, 2\} \\
g_1[j] = 1 & \Rightarrow & g_2[j] \in \{1, 2\} \\
g_2[j] = 2 & \Rightarrow & g_2[j] \in \{0, 1, 2\}
\end{eqnarray}
\end{subequations}
\end{definition}

The same concept can be expressed also between a genotype and a haplotype as in the following definition.

\begin{definition}[Compatibility between genotypes and haplotypes]
Let $g$ be a genotype and $h$ a haplotype, $g$ and $h$ are \emph{compatible} if, for all $j = 1, \ldots, m$, the following conditions hold:
\begin{subequations}
\begin{eqnarray}
g[j] = 0 & \Rightarrow & h[j] = 0 \\
g[j] = 1 & \Rightarrow & h[j] = 1 \\
g[j] = 2 & \Rightarrow & h[j] \in \{0, 1 \}
\end{eqnarray}
\end{subequations}
We denote this relation with $h \compatible g$. Moreover with an abuse of notation we indicate with $h \compatible \{g_1, g_2, \dots \}$ the set of all the genotypes that are compatible with haplotype $h$.
\end{definition}

Observe that the set of compatible genotypes of a haplotype can contain only mutually compatible genotypes (i.e., they form a clique in the compatibility graph).

Another interesting observation is the following. Due to the resolution definition, when one of the two haplotypes composing the pair, say $h$, has been selected, then the other haplotype can be directly inferred from $h$ and the genotype $g$ thanks to the resolution conditions.

\begin{proposition}[Haplotype complement]
Given a genotype $g$ and a haplotype $h \compatible g$, there exists a unique haplotype $k$ such that $h \oplus k = g$. The haplotype $k$ is called the \emph{complement} of $h$ with respect to $g$ and is denoted with $k = g \ominus h$.
\end{proposition}

\begin{proof}
The existence and uniqueness of $k$ is a direct consequence of Conditions~(\ref{eq:chromosome0})--(\ref{eq:chromosome2}). 
\end{proof}

\section{Local Search models for \hpp{}}
\label{sec:local-search}

We start our conceptual analysis of metaheuristic approaches for \hpp{} with the basic building blocks of local search methods. 
Indeed, in order to apply this class of methods to a given problem we need to specify three entities, namely the \emph{search space}, the \emph{cost function} and the \emph{neighborhood relation}, that constitute the so-called \emph{local search model} of the problem.
 
In the following subsections we consider some natural choices for the search space and the cost function, and we try to envisage their benefits and drawbacks on the basis of the general experience on different problem domains. Instead, the discussion of the possible neighborhood relations is provided in the next section together with the general strategies for tackling the problem.

\subsection{Search space}

As the search space for the \hpp{} problem we could either adopt a \emph{complete} or \emph{incomplete} representation as explained below.

\begin{description}
\item[Complete representation:] This search space is based on a straightforward encoding of the problem. In this representation we consider, for each genotype $g$, the pair of haplotypes $\langle h, k \rangle$ that resolves it. Therefore in this representation all the genotypes are fully resolved at each state by construction. The search space is therefore the collection of sets $R$ defined as in the problem statement.

Notice that, thanks to complementarity w.r.t. a genotype just one haplotype can be selected to represent a pair of resolvents. Moreover, since the pair is unordered, the haplotype pair $\langle h, k \rangle$ is equal to $\langle k, h \rangle$, therefore we can break this symmetry by selecting the haplotype $h$ as the representative for the state if $h \prec k$, where the symbol $\prec$ indicates the lexicographic precedence of the two strings, otherwise we select $k$ as the representative.

\item[Incomplete representation:]
In the incomplete representation, instead, we deal with sets of haplotypes that are the elements of the set $H$. An element $h \in H$ is a \emph{potential resolvent} of a genotype $g$, and it actually resolves it only whether its companion $k$, so that $\langle h, k \rangle \resolve g$, also belongs to $H$. As a consequence, in this representation not necessarily all the genotypes have a resolvent and the search space is the powerset $\mathcal{P}(\bigcup_{i = 1}^{n} A_{i})$, where $A_{i} = \{ h | \exists k \langle h, k \rangle \resolve g_i \}$ is the set of all potential resolvents of a genotype $g_i$ of the problem.
This formulation is an incomplete representation of the problem, since the genotype resolution is only potential and a solution of the problem must be constructed on the basis of the haplotypes selected in a given state.
\end{description}


The complete representation is in general more adequate for hybridization with constructive methods since it maintains the feasibility (i.e., the resolution of all genotypes) at every state of the search. 
For example it makes it possible to employ local search in a GRASP-like manner \cite{ReRi03} or to hybridize it with more informed strategies making use problem-specific knowledge.

On the other hand, in some problem domains, this type of representation could constrain too much the search so that it is difficult to explore different areas of the search space because of the impossibility to relax feasibility constraints. For this reason, sophisticated metaheuristic approaches must be employed to improve the effectiveness of methods based on this representation.

Moreover, this formulation makes the neighborhood design more difficult since the search can move only from and to fully resolved sets of genotypes. A possibility to deal with this complication could be the fine-grain hybridization with declarative neighborhood formulations that explicitly encode the feasibility constraints (e.g., those expressed through Constraint Programming or Integer Linear Programming).

On the contrary, the incomplete representation does not guarantee feasibility at every search step since in a given state it is not mandatory to resolve all the genotypes. As a consequence, it commonly happens that in this kind of search spaces a great part of the search effort focuses in uninteresting areas of the search space just to pursue a feasible solution. 

Nevertheless, this representation permits more freedom to explore different areas of the search space and a higher degree of flexibility. For example, it allows an easier design of neighborhood relations.

\subsection{Cost function}
\label{sec:cost-function}

As for the cost function we identify different components related either to optimality or to feasibility of the problem.

A natural component is the objective function of the original problem, that is the cardinality $|H|$ of the set of haplotypes employed in the resolution. 

\begin{equation}
f_1 = |H|
\end{equation}

In the incomplete representation, however, we must also include the so-called \emph{distance to feasibility}, which can simply be measured as the number of genotypes not resolved in the current solution.

\begin{equation}
f_2 = |\{g| g\,\textrm{is not resolved from haplotypes in}\,H\}|
\end{equation}

Moreover in both formulations we might want to include some heuristic measure related to the potential quality of the solution. To this respect, a possible measure could be the number of genotypes resolved by each haplotype.

\begin{equation}
f_3 = \sum_{h \in H} |\{g| g\,\textrm{can be resolved by haplotype}\,h\}|
\end{equation}

Components $f_1$ and $f_2$ have to be minimized in the cost function whereas, in contrast, $f_3$ must be maximized. To simplify the formulation we prefer to deal only with minimization components, therefore we consider the function $f'_3 = n|H| - f_3$ ($n|H|$ is an upper bound on function $f_3$) as the component to be included in the cost function.

The cost function $F$ is then the weighted sum of the three components:
\begin{equation}
F = \alpha_1 f_1 + \alpha_2 f_2 + \alpha_3 f'_3 
\end{equation}
in which the weights $\alpha_1$, $\alpha_2$ and $\alpha_3$ must be chosen for the problem at hand to reflect the trade-offs among the different components.

Notice that $\alpha_2$ can be zero in the case of the complete representation of the search space. Conversely, $\alpha_2$ must be a value much greater than the possible values of the other components in the case of the incomplete representation in order to reflect the importance of feasibility.

\section{Metaheuristic approaches}
\label{sec:metaheuristic}

In this section we discuss some possible metaheuristic approaches to tackle the
problem and we emphasize their strengths and weaknesses. 



In the following discussion we will assume states encoded with the incomplete formulation, even if most of the considerations holds also for the complete formulation. Indeed, the difference between the two is that the latter is associated to a cost function in which the distance from a feasible solution is not considered, as all states are feasible.

In general, the problem can be solved either by trying to solve
it for a given cardinality  $k = |H|$, progressively decreased during search, or
by trying to minimize $|H|$ while searching at the same time a set of haplotypes
that resolve all the genotypes. The first approach is often preferred over the
former as it can be easily reduced to a sequence of feasibility problems (see,
for example, the case of graph coloring problem family~\cite{HoSt05}).
Nevertheless, this method might require a very high execution time, because a
quite long series of feasibility problems has often to be solved. In the next
subsections we detail some promising formulations, using either of the mentioned
approaches, which can be effectively employed by metaheuristic strategies such
as Tabu Search, Iterated local search, Variable neighborhood search,
etc.~\cite{BluRol2003:acm-comp-sur}. Indeed, the choice of the actual
metaheuristic is in general orthogonal with respect to the formulation.

\subsection{Iterative minimization of the number of haplotypes}

This approach is based on the classical method of iteratively solving the problem by decreasing values of $k = |H|$.  Hence, the problem is reduced to find $k$ different haplotypes such that all the genotypes can be resolved. A solution to the problem is represented by a set of distinct haplotypes. Neighborhoods can be defined in many ways; in the following we discuss what we believe are the most promising ones.

\subsubsection{Hamming neighborhoods}

As it is often the case with solutions encoded as boolean variables, also in this case we can consider a neighborhood based on Hamming distance. In particular, a good trade-off between exploration and execution time is the 1-Hamming distance neighborhood w.r.t. each haplotype in the current solution. The exploration of such a neighborhood has a time complexity of $O(nk)$, where $k$ is the number of haplotypes and $n$ the number of sites per haplotype. Besides its low time complexity, this neighborhood has the advantage of making it very easy to generate more complex neighborhoods just by concatenating 1-Hamming distance moves. The drawback of this choice relies in its local character of exploring neighboring solutions, that might not help search escape from areas in which it has been attracted.

\subsubsection{Deletion/insertion}

A complementary neighborhood structure is the one that does not pose restrictions on the distance among solutions. The neighborhoods based on deletion/insertion consist in deleting an haplotype in the solution and inserting a new one. The criteria upon which haplotypes are deleted and inserted define the actual neighborhood; moreover, the way the neighborhood is explored represents another degree of freedom.

The most general neighborhood relation would be to try the move deletion/insertion for all the possible combinations, i.e., for each haplotype in the solution, try to substitute it with each possible new haplotype and check the corresponding cost function. Since the number of haplotypes with $n$ sites is $2^n$, this choice is in general not feasible. A more practical solution, instead, consists of choosing the new haplotype among the ones stored in a kind of candidate set of given cardinality. Among the possible ways of generating the candidate set we may distinguish between random and heuristically generated. The general neighborhood relation can thus be defined as follows:

\begin{equation}
N(s) = \{M | M = N(s) \setminus s' \cup z, s' \in N(s), z \in C\}
\end{equation}
where $s$ is the current solution and $C$ the candidate set. Both the elements and cardinality of $C$ can be kept constant or varied during search. Since the complexity of enumerating all the neighborhood is $O(|H|\cdot |C|)$, an efficient balance between exploration and time complexity has to be set.

\subsubsection{A hybrid strategy: Iterated local search}

The characteristics of the previous described neighborhoods can be combined in a single metaheuristic such as Iterated local search~\cite{LMS01}. This local search metaheuristic uses the first neighborhood (1-Hamming) with a local search metaheuristic (e.g., hill climbing, tabu search, etc.) until stagnation, then a run of another local search method with the insertion/deletion neighborhood is performed in order to move away from the attractor (perturbation phase). The strategy is described in algorithm~\ref{alg:2neighILS}.

\begin{algorithm}
\caption{Two neighborhoods Iterated Local Search}
\label{alg:2neighILS}
\begin{algorithmic}
        \STATE $s_0 \leftarrow$  \textsf{GenerateInitialSolution()} 
        \STATE $s^{*} \leftarrow$ \textsf{LocalSearch($s_0$) with 1-Hamming neighborhood}
        \WHILE {termination conditions not met}
                \STATE $s' \leftarrow$ \textsf{LocalSearch($s^{*}$) with I/D neighborhood} \COMMENT{perturbation phase}
                \STATE $s^{*'} \leftarrow$ \textsf{LocalSearch($s'$) with 1-Hamming neighborhood}
                \STATE $s^{*} \leftarrow$ \textsf{ApplyAcceptanceCriterion($s^{*},s^{*'},history$)}
        \ENDWHILE
\end{algorithmic}
\end{algorithm}

\subsection{Minimization and feasibility approach}

The second approach for tackling the \hpp{} problem defines a search strategy that tries to minimize $|H|$ and resolve all the genotypes at the same time. In such a case, it is possible that some genotypes are not resolved during search, therefore also states which are infeasible w.r.t. the original problem formulations are explored during search. We will illustrate two possible strategies for implementing metaheuristics based on this problem formulation.

\subsubsection{Dynamic local search}

The core of a possible local search strategy for this model strongly relies on an effective penalty mechanism that should be able to guide the search toward `good' (i.e., at low values of $|H|$) and feasible (i.e., all genotypes are resolved) states. A very effective technique for achieving such a search guidance is what is usually called \emph{Dynamic local search}~\cite{HoSt05}, in which the cost function is composed of two (or more) components, each with a weight that is dynamically varied during search. In formulas, referring to the cost functions defined in section~\ref{sec:cost-function}:

\begin{eqnarray}
F & = & w_1 f_1 + w_2 f_2 \\ 
& = & w_1 |H| + w_2 |\{g| g\,\textrm{is not resolved from haplotypes in}\,H\}| \nonumber
\end{eqnarray}

The coefficients $w_1$ and $w_2$ are varied during search by trying to explore feasible states with $|H|$ as low as possible. Effective techniques for achieving this goal can be inspired by Guided local search, as discussed in~\cite{BluRol2003:acm-comp-sur} and, more in general, Dynamic local search strategies described in~\cite{HoSt05}. The high level description of this metaheuristic is reported in algorithm~\ref{alg:DLS}. A very effective scheme used for weight update is the \emph{shifting penalty} method, that, in this case, can be implemented as follows. First of all, $w_1$ and $w_2$ have to be initialized in such a way that the contribution of $w_2 f_2$ (feasibility) is higher than the optimization contribution of $w_1 f_1$. Then, during search, if the set of haplotypes resolve all the genotypes for $K$ consecutive iterations, $w_2$ is reduced (by a small random factor $\gamma > 1$), otherwise it is increased (by a comparable value).

\begin{algorithm}
\caption{Dynamic Local Search for \hpp{} problem}
\label{alg:DLS}
\begin{algorithmic}
\STATE $s \leftarrow$ {\sf GenerateInitialSolution()}
\WHILE {termination conditions not met}
        \STATE $s \leftarrow$ {\sf LocalSearch}($s$)
        \STATE {\sf Update}(weights $w_1$, $w_2$)
\ENDWHILE
\end{algorithmic}
\end{algorithm}

\subsubsection{Adaptive constructive approach}

Another very promising approach, which deals with both complete and incomplete representation of the search space, is inspired by Strategic oscillation~\cite{GL97} and Squeaky wheel search~\cite{JC99}. The idea is to expand and reduce the set of haplotypes in solution by trying to minimize the combined cost function $F$ defined above.
The difference between dynamic local search is that the focus of this algorithm is based on the selection of elements to be deleted and to be inserted, i.e., the critical decision is on the delete and insert sets. Several effective techniques can be employed, such as the already mentioned Strategic oscillation and Squeaky wheel search; moreover, some incomplete iterative constructive heuristics could be effectively applied, such as Incomplete dynamic backtracking~\cite{prestwich2002-idb}. 

As an example, we briefly discuss one of the possible implementations of such a strategy for tackling the \hpp{} problem. In this implementation, haplotypes are added one at a time and we use a probabilistic criterion for accepting the new haplotype. As shown in algorithm~\ref{alg:AC}, we start with an empty set of haplotypes and we incrementally add one element at a time. The criteria upon which the candidate set is built and the new haplotype is chosen can be based on heuristics or randomness or both. The new element is always added to the current solution if it decreases the cost function (i.e., the overall quality of the solution increases); even if it does not improve the current solution, it can be added anyway in a simulated annealing fashion. 

When the cardinality of $|H|$ reaches its maximum value (parameter of the algorithm) some haplotypes are removed from the solution and the insertion restarted on this smaller set. As one can easily see, there are many choices a designer has to make and the implementation and comparison of these techniques is subject of ongoing work.

\begin{algorithm}
\caption{Adaptive constructive metaheuristic for the \hpp{} problem}
\label{alg:AC}
\begin{algorithmic}
\STATE $s \leftarrow \emptyset$ \COMMENT{Initial solution is an empty set} 
\WHILE {termination conditions not met}
        \STATE $s' \leftarrow$ $s \cup h'$, $h' \in \rm{C}$ \COMMENT{Choose haplotype to add from the candidate set}
	\STATE $\Delta \leftarrow F(s') - F(s)$
        \IF {$\Delta < 0$}
		\STATE $s \leftarrow s'$
	\ELSE
		\STATE With probability $p(\Delta,history)$: $s \leftarrow s'$
	\ENDIF
	\IF {$|H| > maxCard$}
		\STATE Delete $q$ haplotypes from $s$
	\ENDIF
\ENDWHILE
\end{algorithmic}
\end{algorithm}

\section{Constructive procedures for initial solutions}
\label{sec:constructive}

Constructive heuristics produce often an important contribution to the effectiveness of the metaheuristic. Initial sets of haplotypes can be constructed on the basis of several heuristics. For instance, variations of the Clark's algorithm and the other rule-based methods described in~\cite{Gusf01} can be employed. Moreover, some SAT/Constraint Programming-based techniques can be used to reduce the search space for the initial solution construction. 

One critical factor the designer has to take into account is whether to prefer small $|H|$ over number of resolved genotypes or viceversa. In the first case, a better choice would be to start from an empty set, while in the second a complete set (corresponding to the complete formulation) would be preferable.
Furthermore, constructive procedures should also be employed for generating the candidate sets, as this could make the search more effective.

\section{Discussion and future work}
\label{sec:discussion}

We have presented a feasibility study on the application of metaheuristics to the \hpp{} problem. The main purpose of this work was to point out critical design issues about the problem in order to guide future developments and to foster further research on metaheuristic approaches to this problem.
Indeed, we believe that the \hpp{} problem could become a relevant problem subject of application of metaheuristic techniques. However, besides the relevance of the \hpp{} problem itself, this preliminary analysis has posed some challenges in modeling hybrid metaheuristic solver design that can be generalized to other problems.

The result of this analysis is in some sense a roadmap of the work that has still to be done. The first steps consist in implementing the described metaheuristic solvers for the problem and evaluating them on some benchmarks. To this respect we plan to use a set of publicly available datasets (simulated or extracted from the HapMap project), which are provided by J. Marchini at the website \url{http://www.stats.ox.ac.uk/~marchini/phaseoff.html}. 

Afterwards, we will investigate the ways in which domain specific knowledge can be integrated in the metaheuristics. As an example, it is possible to easily construct (in)compatibility relations between genotypes and it will be interesting to exploit this information to make the search more effective. Other ways to integrate domain specific information is through the hybridization with constructive techniques or with other search paradigms (e.g., Constraint Programming \cite{Apt03}).

Finally, in future work we intend also to take into consideration different kind of metaheuristics such as population-based algorithms (e.g., Ant-Colony Optimization \cite{DoSt04} or Memetic algorithms \cite{Mosc89}).

Some of the metaheuristics described in this paper have been already implemented and subject of ongoing testing. We are currently designing and implementing advanced metaheuristics for the complete and incomplete formulation.

\section*{Appendix: Further observations on the problem}

To the best of our knowledge, there have been no attempts to exploit structural properties of the problem which can be deduced from compatibility graphs, or other problem representations. In this section, we present a {\it{reduction procedure}} that starts from a set of haplotypes in the complete representation and tries to reduce its cardinality by exploiting compatibility properties of the instance. Other heuristics based on graph representation of the problem are subject of ongoing work.

\subsection{Haplotype cardinality reduction}

Let us illustrate this property with an example. Consider the following set of genotypes, which corresponds to the compatibility graph in Figure~\ref{fig:ex-compatibility}.

\noindent
\begin{tabular}{*{3}{p{0.27\textwidth}}}
$g_1: (2210212)$ & $g_3: (1212122)$ & $g_5: (1202201)$ \\
$g_2: (2112110)$ & $g_4: (1222122)$ & ~ \\
\end{tabular}

\begin{figure}[t]
\begin{center}

\begin{picture}(56,32)(0,-32)

\node[NLangle=0.0](n0)(0.0,-16.0){$g_1$}
\node[NLangle=0.0](n1)(16.0,-0.0){$g_2$}
\node[NLangle=0.0](n2)(16.0,-32.0){$g_3$}
\node[NLangle=0.0](n3)(32.0,-16.0){$g_4$}
\node[NLangle=0.0](n4)(56.0,-16.0){$g_5$}

\drawedge[dash={3.0 5.0}{0.0},AHnb=0](n0,n1){ }
\drawedge[dash={3.0 5.0}{0.0},AHnb=0](n1,n3){ }
\drawedge[dash={3.0 5.0}{0.0},AHnb=0](n3,n4){ }
\drawedge[dash={3.0 5.0}{0.0},AHnb=0](n3,n2){ }
\drawedge[dash={3.0 5.0}{0.0},AHnb=0](n2,n0){ }
\drawedge[dash={3.0 5.0}{0.0},AHnb=0](n0,n3){ }
\drawedge[dash={3.0 5.0}{0.0},AHnb=0](n1,n2){ }
\drawedge[dash={3.0 5.0}{0.0},AHnb=0](n1,n0){ }

\end{picture}

\end{center}
\caption{Compatibility graph}
\label{fig:ex-compatibility}
\end{figure}
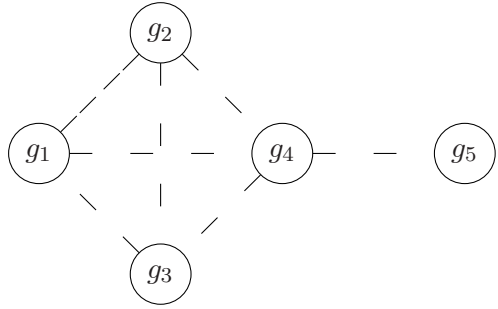

Now consider the following set of haplotypes that resolves all the genotypes $g_1, \ldots, g_5$. The set consists of 8 distinct haplotypes; genotypes currently resolved by a given haplotype are denoted by bold typeface in the resolvent list (i.e., we write $\textbf{g}$ if $g \hresolved h \in H$).

\noindent
\begin{tabular}{*{2}{p{0.5\textwidth}}}
\begin{eqnarray*}
g_1 & \resolved & 
\begin{cases}
a = (1110110) & \compatible \{\mathbf{g_1}, \mathbf{g_2}, g_3, g_4\} \\
b = (0010010) & \compatible \{\mathbf{g_1}\} \\
\end{cases} \\
g_2 & \resolved & 
\begin{cases}
c = (0111110) & \compatible \{\mathbf{g_2}\} \\
a = (1110110) & \compatible \{\mathbf{g_1}, \mathbf{g_2}, g_3, g_4\} \\
\end{cases} \\
g_3 & \resolved & 
\begin{cases}
d = (1111101) & \compatible \{\mathbf{g_3}, g_4\} \\
e = (1010110) & \compatible \{g_1, \mathbf{g_3}, g_4\} \\
\end{cases} \\
\end{eqnarray*}
&
\begin{eqnarray*}
g_4 & \resolved & 
\begin{cases}
f = (1101101) & \compatible \{\mathbf{g_4}, \mathbf{g_5}\} \\
p = (1010101) & \compatible \{g_3, \mathbf{g_4}\} \\
\end{cases} \\
g_5 & \resolved & 
\begin{cases}
f = (1101101) & \compatible \{\mathbf{g_4}, \mathbf{g_5}\} \\
q = (1000001) & \compatible \{\mathbf{g_5}\} \\
\end{cases} \\
\end{eqnarray*}
\end{tabular}

Notice that the haplotype $a = (1110110)$, resolving $g_1$, is compatible also with genotypes $g_2$, $g_3$ and $g_4$ and, as a matter of fact, it resolves $g_2$. This configuration is depicted in the bipartite graph of Figure~\ref{fig:pre-resolution} (dashed edges between genotypes represent genotype compatibility) in which the haplotypes are represented by square nodes, the compatibility between haplotypes and genotypes is represented by solid edges and a bold edge represents the current resolution of a genotype by a haplotype. We will call this graph {\it{extended (compatibility) graph}}. The constraint on genotypes resolution is mapped onto the extended graph by imposing that every genotype node must have (at least) two bold edges. The goal of the reduction procedure is to try to decrease the number of distinct haplotypes, i.e., the number of square nodes while satisfying the resolution constraint. The intuition behind the procedure is that a possible way of reducing the haplotype number is to resolve a genotype by a haplotype that is compatible, but not currently resolving it, i.e., changing an arc from solid to bold. Of course, this move must be followed by repairing moves in the graph so that the state is still feasible.These moves consists in adding one or more haplotypes and relinking some nodes.

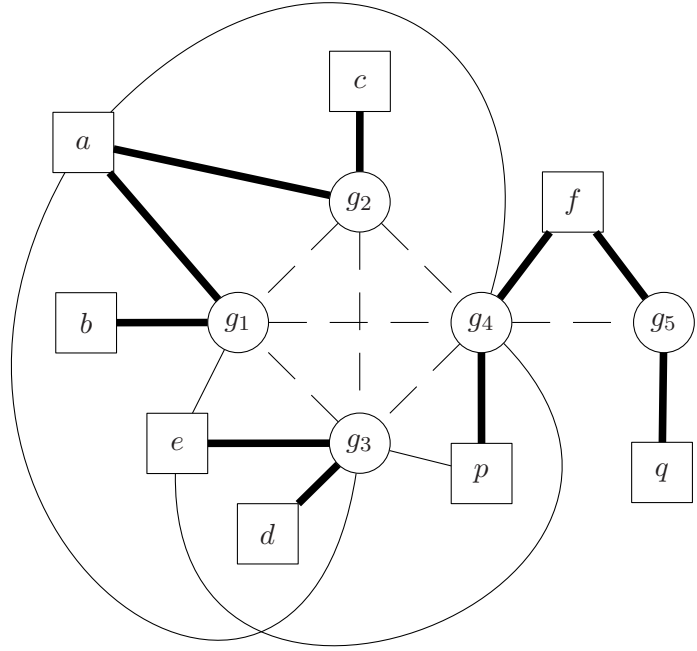
\begin{figure}[t]
\begin{center}

\begin{picture}(119,90)(0,-90)

\node[NLangle=0.0](n0)(36.25,-44.12){$g_1$}
\node[NLangle=0.0](n1)(52.25,-28.12){$g_2$}
\node[NLangle=0.0](n2)(52.25,-60.12){$g_3$}
\node[NLangle=0.0](n3)(68.25,-44.12){$g_4$}
\node[NLangle=0.0](n4)(92.25,-44.12){$g_5$}
\node[NLangle=0.0,Nmr=0.0](n5)(15.89,-20.24){$a$}
\node[NLangle=0.0,Nmr=0.0](n6)(16.26,-44.16){$b$}
\node[NLangle=0.0,Nmr=0.0](n7)(28.25,-60.15){$e$}
\node[NLangle=0.0,Nmr=0.0](n8)(40.15,-72.14){$d$}
\node[NLangle=0.0,Nmr=0.0](n9)(52.27,-12.14){$c$}
\node[NLangle=0.0,Nmr=0.0](n10)(68.27,-64.14){$p$}
\node[NLangle=0.0,Nmr=0.0](n12)(80.27,-28.14){$f$}
\node[NLangle=0.0,Nmr=0.0](n13)(92.0,-64.0){$q$}

\drawedge[dash={5.0 3.0}{0.0},AHnb=0](n1,n3){ }
\drawedge[dash={5.0 3.0}{0.0},AHnb=0](n3,n4){ }
\drawedge[dash={5.0 3.0}{0.0},AHnb=0](n3,n2){ }
\drawedge[dash={5.0 3.0}{0.0},AHnb=0](n2,n0){ }
\drawedge[dash={5.0 3.0}{0.0},AHnb=0](n0,n3){ }
\drawedge[dash={5.0 3.0}{0.0},AHnb=0](n1,n2){ }
\drawedge[dash={5.0 3.0}{0.0},AHnb=0](n1,n0){ }

\drawedge[linewidth=1.0,AHnb=0](n6,n0){ }
\drawbpedge[AHnb=0](n5,50,47.11,n3,70,41.17){ }
\drawbpedge[AHnb=0](n5,-123,61.05,n2,-95,56.69){ }
\drawedge[linewidth=1.0,AHnb=0](n5,n1){ }
\drawedge[linewidth=1.0,AHnb=0](n5,n0){ }
\drawedge[AHnb=0](n7,n0){ }
\drawedge[linewidth=1.0,AHnb=0](n7,n2){ }
\drawbpedge[AHnb=0](n7,-96,52.16,n3,-41,52.73){ }
\drawedge[linewidth=1.0,AHnb=0](n8,n2){ }
\drawedge[linewidth=1.0,AHnb=0](n9,n1){ }
\drawedge[linewidth=1.0,AHnb=0](n10,n3){ }
\drawedge[AHnb=0](n10,n2){ }
\drawedge[linewidth=1.0,AHnb=0](n12,n4){ }
\drawedge[linewidth=1.0,AHnb=0](n12,n3){ }
\drawedge[linewidth=1.0,AHnb=0](n13,n4){ }

\end{picture}

\end{center}
\caption{Haplotype resolution graph}
\label{fig:pre-resolution}
\end{figure}

From the situation described above, we can use the resolvent of $g_1$ to resolve $g_3$ or $g_4$, however the situation is different in the two cases. For example, if we use the resolvent $a$ of $g_1$ to resolve $g_3$, then the situation will become:

\begin{eqnarray*}
g_3 & \resolved & 
\begin{cases}
a = (1110110) & \compatible \{\mathbf{g_1}, \mathbf{g_2}, \mathbf{g_3}, g_4 \} \\
r = (1011101) & \compatible \{\mathbf{g_3}, g_4 \} \\
\end{cases}
\end{eqnarray*}

As an effect of the reduction, the total number of haplotypes employed in the solution has been decreased by 1, since now haplotypes $d$ and $e$ resolving $g_3$ are replaced by $a$, which was already a member of the haplotype set, and a new haplotype $r$ (see Figure~\ref{fig:reduction-g3}).

\begin{figure}[t]
\begin{center}

\begin{picture}(119,90)(0,-90)

\node[NLangle=0.0](n0)(36.25,-44.12){$g_1$}
\node[NLangle=0.0](n1)(52.25,-28.12){$g_2$}
\node[NLangle=0.0](n2)(52.25,-60.12){$g_3$}
\node[NLangle=0.0](n3)(68.25,-44.12){$g_4$}
\node[NLangle=0.0](n4)(92.25,-44.12){$g_5$}

\node[NLangle=0.0,Nmr=0.0](n5)(15.89,-20.24){$a$}
\node[NLangle=0.0,Nmr=0.0](n6)(16.26,-44.16){$b$}
\node[linecolor=Gray,linewidth=0.5,NLangle=0.0,Nmr=0.0](n8)(84.0,-80.0){$r$}
\node[NLangle=0.0,Nmr=0.0](n9)(52.27,-12.14){$c$}
\node[NLangle=0.0,Nmr=0.0](n10)(68.27,-64.14){$p$}
\node[NLangle=0.0,Nmr=0.0](n12)(80.27,-28.14){$f$}
\node[NLangle=0.0,Nmr=0.0](n13)(92.0,-64.0){$q$}

\drawedge[dash={3.0 5.0}{0.0},AHnb=0](n1,n3){ }
\drawedge[dash={3.0 5.0}{0.0},AHnb=0](n3,n4){ }
\drawedge[dash={3.0 5.0}{0.0},AHnb=0](n3,n2){ }
\drawedge[dash={3.0 5.0}{0.0},AHnb=0](n2,n0){ }
\drawedge[dash={3.0 5.0}{0.0},AHnb=0](n0,n3){ }
\drawedge[dash={3.0 5.0}{0.0},AHnb=0](n1,n2){ }
\drawedge[dash={3.0 5.0}{0.0},AHnb=0](n1,n0){ }

\drawedge[linewidth=1.0,AHnb=0](n6,n0){ }
\drawbpedge[AHnb=0](n5,49,47.11,n3,70,41.17){ }
\drawbpedge[linewidth=1.0,AHnb=0](n5,-122,61.05,n2,-95,56.69){ }
\drawedge[linewidth=1.0,AHnb=0](n5,n1){ }
\drawedge[linewidth=1.0,AHnb=0](n5,n0){ }

\drawedge[linewidth=1.0,AHnb=0](n9,n1){ }
\drawedge[linewidth=1.0,AHnb=0](n10,n3){ }
\drawedge[AHnb=0](n10,n2){ }

\drawedge[linewidth=1.0,AHnb=0](n12,n4){ }
\drawedge[linewidth=1.0,AHnb=0](n12,n3){ }
\drawedge[linewidth=1.0,AHnb=0](n13,n4){ }
\drawedge[AHnb=0](n8,n3){ }
\drawedge[linewidth=1.0,AHnb=0,curvedepth=8.0](n8,n2){ }

\end{picture}

\end{center}
\caption{The haplotype resolution graph after the reduction of $g_3$}
\label{fig:reduction-g3}
\end{figure}
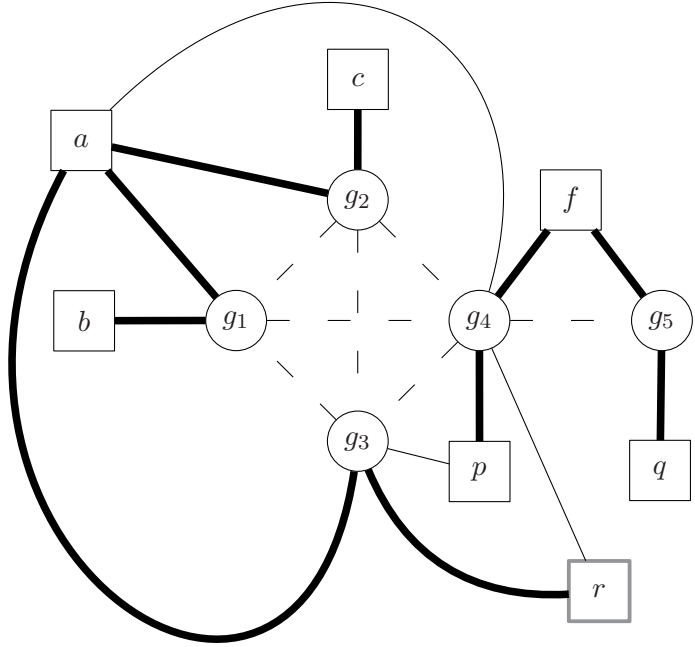

Instead, if we use the resolvent $a$ to resolve $g_4$ we obtain the following situation (see Figure~\ref{fig:reduction-g4}):

\begin{eqnarray*}
g_4 & \resolved & 
\begin{cases}
a = (1110110) & \compatible \{\mathbf{g_1}, \mathbf{g_2}, \mathbf{g_3}, \mathbf{g_4} \} \\
s = (1001101) & \compatible \{\mathbf{g_4} \} \\
\end{cases}
\end{eqnarray*}

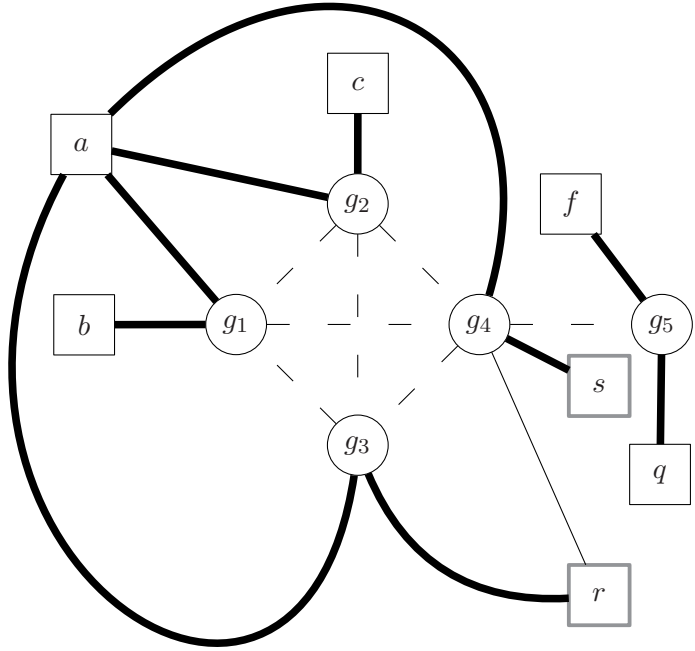
\begin{figure}[t]
\begin{center}

\begin{picture}(119,90)(0,-90)

\node[NLangle=0.0](n0)(36.25,-44.12){$g_1$}
\node[NLangle=0.0](n1)(52.25,-28.12){$g_2$}
\node[NLangle=0.0](n2)(52.25,-60.12){$g_3$}
\node[NLangle=0.0](n3)(68.25,-44.12){$g_4$}
\node[NLangle=0.0](n4)(92.25,-44.12){$g_5$}

\node[NLangle=0.0,Nmr=0.0](n5)(15.89,-20.24){$a$}
\node[NLangle=0.0,Nmr=0.0](n6)(16.26,-44.16){$b$}
\node[linecolor=Gray,linewidth=0.5,NLangle=0.0,Nmr=0.0](n8)(84.0,-80.0){$r$}
\node[NLangle=0.0,Nmr=0.0](n9)(52.27,-12.14){$c$}
\node[linecolor=Gray,linewidth=0.5,NLangle=0.0,Nmr=0.0](n10)(84.0,-52.26){$s$}
\node[NLangle=0.0,Nmr=0.0](n12)(80.27,-28.14){$f$}
\node[NLangle=0.0,Nmr=0.0](n13)(92.0,-64.0){$q$}

\drawedge[dash={3.0 5.0}{0.0},AHnb=0](n1,n3){ }
\drawedge[dash={3.0 5.0}{0.0},AHnb=0](n3,n4){ }
\drawedge[dash={3.0 5.0}{0.0},AHnb=0](n3,n2){ }
\drawedge[dash={3.0 5.0}{0.0},AHnb=0](n2,n0){ }
\drawedge[dash={3.0 5.0}{0.0},AHnb=0](n0,n3){ }
\drawedge[dash={3.0 5.0}{0.0},AHnb=0](n1,n2){ }
\drawedge[dash={3.0 5.0}{0.0},AHnb=0](n1,n0){ }
\drawedge[linewidth=1.0,AHnb=0](n6,n0){ }
\drawbpedge[linewidth=1.0,AHnb=0](n5,49,47.11,n3,70,41.17){ }
\drawbpedge[linewidth=1.0,AHnb=0](n5,-122,61.05,n2,-95,56.69){ }
\drawedge[linewidth=1.0,AHnb=0](n5,n1){ }
\drawedge[linewidth=1.0,AHnb=0](n5,n0){ }
\drawedge[linewidth=1.0,AHnb=0](n9,n1){ }
\drawedge[linewidth=1.0,AHnb=0](n10,n3){ }
\drawedge[linewidth=1.0,AHnb=0](n12,n4){ }
\drawedge[linewidth=1.0,AHnb=0](n13,n4){ }
\drawedge[AHnb=0](n8,n3){ }
\drawedge[linewidth=1.0,AHnb=0,curvedepth=8.0](n8,n2){ }

\end{picture}

\end{center}
\caption{The haplotype resolution graph after the reduction of $g_4$}
\label{fig:reduction-g4}
\end{figure}

Differently from the previous reduction, in this case the number of haplotypes in the solution has not changed. The reason of this is that the haplotype $f$ was already shared between the resolutions of $g_4$ and $g_5$, therefore the reduction operation removes one haplotype but it introduces also a new one, leaving the total number of haplotypes unchanged.

The situation of the example can be generalized in the following proposition.

\begin{proposition}[Haplotype local reduction]
Given $n$ genotypes $G = \{g_1, \ldots, g_n\}$ and the resolvent set $R = \{ \langle h_1, b \rangle, \ldots, \langle h_n, k_n \rangle\}$, so that $\langle h_i, k_i \rangle \resolve g_i$. Suppose there exist two genotypes $g, g' \in G$ such that:

\noindent
\begin{equation}
g \resolved 
\begin{cases}
h & \compatible \{ g, g', \ldots \} \\
k & \compatible \{g, \ldots \} \\
\end{cases}
, \quad
g' \resolved
\begin{cases}
h' & \compatible \{g', \ldots \} \\
k' & \compatible \{g', \ldots \} \\
\end{cases}
\end{equation}

\noindent
and $h \neq h'$, $h \neq k'$, $h' \hresolve A$, $k' \hresolve B$.

The replacement of $\langle h', k' \rangle$ with $\langle h, g' \ominus h \rangle$ in the resolution of $g'$ is a correct resolution that employs a number of distinct haplotypes according to the following criteria:
\begin{itemize} 
\item if $|A| = 1$ and $|B| = 1$, the new resolution uses \emph{at most} one less distinct haplotype;
\item if $|A| > 1$ and $|B| = 1$ (or symmetrically, $|A| = 1$ and $|B| > 1$), the new resolution uses \emph{at most} the same number of distinct haplotypes;
\item in the remaining case the new resolution uses \emph{at most} one more distinct haplotype.
\end{itemize}
\end{proposition}

\begin{proof}
The proof of the proposition is straightforward. The resolution is obviously correct because $h$ is compatible with $g'$ and $g' \ominus h$ is the complement of $h$ with respect to $g'$. 

Concerning the validity of the conditions on the cardinality, let us proceed by cases and first consider the situation in which $g' \ominus h$ does not resolve any other genotype but $g'$.

If $|A| = |B| = 1$, then $h'$ and $k'$ are not shared with other genotype resolutions so they will not appear in the set $H$ after the replacement, therefore since in the new resolution $h$ is shared between $g$ and $g'$ the cardinality of $H$ is decreased by one. 

Conversely, if one of the sets $|A|$ or $|B|$ consists of more than a genotype and the other set of just one genotype, there is no guarantee of obtaining an improvement from the replacement. Indeed, since one of the two haplotypes is already shared with another genotype there is just a replacement of the shared haplotype with another one in the set $H$. 

Finally, when $|A| > 1$ and $|B| > 1$ both $h'$ and $k'$ are shared with other genotypes therefore the replacement introduces the new haplotype $g' \ominus h$ in the set $H$.

Moving to the situation in which $g' \ominus h$ resolves also other genotypes, the same considerations apply; additionally, given that $g' \ominus h$ is already present in $H$, the number of distinct haplotypes employed in the resolution is decreased by one.
For this reason the estimation of the changes of $|H|$ is conservative. 
\end{proof}

\begin{note}
Perhaps some additional observation could be made in the case of homozygous genotypes $g \resolved \langle h,k \rangle$.
\end{note}

A preliminary experimental analysis on the application of the haplotype local reduction shows that the results achieved are very promising.

\bibliography{haplotype}

\begin{thebibliography}{22}
\providecommand{\natexlab}[1]{#1}
\providecommand{\url}[1]{\texttt{#1}}
\expandafter\ifx\csname urlstyle\endcsname\relax
  \providecommand{\doi}[1]{doi: #1}\else
  \providecommand{\doi}{doi: \begingroup \urlstyle{rm}\Url}\fi

\bibitem[Apt(2003)]{Apt03}
Krzysztof~R. Apt.
\newblock \emph{Principles of Constraint Programming}.
\newblock Cambridge University Press, Cambridge, UK, 2003.
\newblock ISBN 978-0-521-82583-0.

\bibitem[Blum and Roli(2003)]{BluRol2003:acm-comp-sur}
Christian Blum and Andrea Roli.
\newblock Metaheuristics in combinatorial optimization: Overview and conceptual
  comparison.
\newblock \emph{ACM Computing Surveys}, 35\penalty0 (3):\penalty0 268--308,
  September 2003.

\bibitem[Brown and Harrower(2006)]{DBLP:journals/tcbb/BrownH06}
Daniel~G. Brown and Ian~M. Harrower.
\newblock Integer programming approaches to haplotype inference by pure
  parsimony.
\newblock \emph{IEEE/ACM Transactions on Computational Biology and
  Bioinformatics}, 3\penalty0 (2):\penalty0 141--154, 2006.

\bibitem[Clark(1990)]{Clark90}
Andrew~G. Clark.
\newblock Inference of haplotypes from {PCR}-amplified samples of diploid
  populations.
\newblock \emph{Molecular Biology and Evolution}, 7:\penalty0 111--122, 1990.

\bibitem[Dorigo and St\"utzle(2004)]{DoSt04}
Marco Dorigo and Thomas St\"utzle.
\newblock \emph{{A}nt {C}olony {O}ptimization}.
\newblock MIT Press, Cambridge, MA, USA, 2004.
\newblock ISBN 978-0-262-04219-2.

\bibitem[Glover and Laguna(1997)]{GL97}
Fred~W. Glover and Manuel Laguna.
\newblock \emph{Tabu Search}.
\newblock Kluwer Academic Publishers, Norwell, MA, USA, 1997.
\newblock ISBN 978-0-7923-9965-0.

\bibitem[Gusfield(2003)]{DBLP:conf/cpm/Gusfield03}
Dan Gusfield.
\newblock Haplotype inference by pure parsimony.
\newblock In Ricardo~A. Baeza-Yates, Edgar Ch{\'a}vez, and Maxime Crochemore,
  editors, \emph{Combinatorial Pattern Matching (CPM 2003), Proceedings of the
  14th Annual Symposium}, volume 2676 of \emph{Lecture Notes in Computer
  Science}, pages 144--155, Berlin-Heidelberg, Germany, 2003. Springer-Verlag.
\newblock ISBN 3-540-40311-6.

\bibitem[Gusfield(2001)]{Gusf01}
Dan Gusfield.
\newblock Inference of haplotypes from sample diploid populations: complexity
  and algorithms.
\newblock \emph{Journal of Computational Biology}, 8:\penalty0 305--323, 2001.

\bibitem[Halld{\'o}rsson et~al.(2002)Halld{\'o}rsson, Bafna, Edwards, Lippert,
  Yooseph, and Istrail]{DBLP:conf/recomb/HalldorssonBELYI02}
Bjarni~V. Halld{\'o}rsson, Vineet Bafna, Nathan Edwards, Ross Lippert, Shibu
  Yooseph, and Sorin Istrail.
\newblock A survey of computational methods for determining haplotypes.
\newblock In Sorin Istrail, Michael~S. Waterman, and Andrew~G. Clark, editors,
  \emph{Computational Methods for SNPs and Haplotype Inference}, volume 2983 of
  \emph{Lecture Notes in Computer Science}, pages 26--47. Springer, 2002.
\newblock ISBN 3-540-21249-3.

\bibitem[Hoos and St\"utzle(2005)]{HoSt05}
H.H. Hoos and T.~St\"utzle.
\newblock \emph{Stochastic Local Search Foundations and Applications}.
\newblock Morgan Kaufmann Publishers, San Francisco, CA (USA), 2005.
\newblock ISBN 978-1-55860-872-6.

\bibitem[Huang et~al.(2005)Huang, Chao, and Chen]{DBLP:conf/sac/HuangCC05}
Yao-Ting Huang, Kun-Mao Chao, and Ting Chen.
\newblock An approximation algorithm for haplotype inference by maximum
  parsimony.
\newblock In Hisham Haddad, Lorie~M. Liebrock, Andrea Omicini, and Roger~L.
  Wainwright, editors, \emph{Proceedings of the 2005 ACM Symposium on Applied
  Computing (SAC 2005)}, pages 146--150. ACM, 2005.
\newblock ISBN 1-58113-964-0.

\bibitem[Joslin and Clements(1999)]{JC99}
D.E. Joslin and D.P. Clements.
\newblock {``Squeaky Wheel''} {O}ptimization.
\newblock \emph{Journal of Artificial Intelligence Research}, 10:\penalty0
  353--373, 1999.

\bibitem[Kalpakis and Namjoshi(2005)]{DBLP:conf/bibe/KalpakisN05}
Konstantinos Kalpakis and Parag Namjoshi.
\newblock Haplotype phasing using semidefinite programming.
\newblock In \emph{BIBE}, pages 145--152. IEEE Computer Society, 2005.
\newblock ISBN 0-7695-2476-1.

\bibitem[Lancia et~al.(2004)Lancia, Pinotti, and
  Rizzi]{DBLP:journals/informs/LanciaPR04}
Giuseppe Lancia, Maria~Cristina Pinotti, and Romeo Rizzi.
\newblock Haplotyping populations by pure parsimony: Complexity of exact and
  approximation algorithms.
\newblock \emph{INFORMS Journal on Computing}, 16\penalty0 (4):\penalty0
  348--359, 2004.

\bibitem[Louren\c{c}o et~al.(2001)Louren\c{c}o, Martin, and St{\"u}tzle]{LMS01}
H.~Ramalhino Louren\c{c}o, O.~Martin, and T.~St{\"u}tzle.
\newblock Iterated {L}ocal {S}earch.
\newblock To appear in: F.~Glover, G.~Kochenberger, {\it Handbook of
  Metaheuristics}, 2001.

\bibitem[Lynce and Marques-Silva(2006{\natexlab{a}})]{DBLP:conf/aaai/LynceM06}
In{\^e}s Lynce and Jo{\~a}o Marques-Silva.
\newblock Efficient haplotype inference with boolean satisfiability.
\newblock In \emph{Proceedings of the 21st National Conference on Artificial
  Intelligence and the Eighteenth Innovative Applications of Artificial
  Intelligence Conference}, Menlo Park, CA, USA, 2006{\natexlab{a}}. AAAI
  Press.

\bibitem[Lynce and Marques-Silva(2006{\natexlab{b}})]{DBLP:conf/sat/LynceM06}
In{\^e}s Lynce and Jo{\~a}o Marques-Silva.
\newblock {SAT} in bioinformatics: Making the case with haplotype inference.
\newblock In Armin Biere and Carla~P. Gomes, editors, \emph{SAT}, volume 4121
  of \emph{Lecture Notes in Computer Science}, pages 136--141,
  Berlin-Heidelberg, Germany, 2006{\natexlab{b}}. Springer-Verlag.
\newblock ISBN 3-540-37206-7.

\bibitem[Moscato(1989)]{Mosc89}
Pablo Moscato.
\newblock On evolution, search, optimization, genetic algorithms and martial
  arts: Towards memetic algorithms.
\newblock Technical Report 826, Caltech Concurrent Computation Program (C3P),
  1989.

\bibitem[Prestwich(2002)]{prestwich2002-idb}
S.~Prestwich.
\newblock Combining the scalability of local search with the pruning techniques
  of systematic search.
\newblock \emph{Annals of Operations Research}, 115:\penalty0 51--72, 2002.

\bibitem[Resende and Ribeiro(2003)]{ReRi03}
Mauricio G.~C. Resende and Celso~C. Ribeiro.
\newblock Greedy randomized adaptive search procedures.
\newblock In Fred Glover and Gary~A. Kocheberger, editors, \emph{Handbook of
  Metaheuristics}, volume~57 of \emph{International Series in Operations
  Research \& Management Science}, pages 219--249. Kluwer Academic Publishers,
  Norwell, MA, USA, 2003.
\newblock ISBN 1-4020-7263-5.

\bibitem[{The International {HapMap} Consortium}(2003)]{HapMap03}
{The International {HapMap} Consortium}.
\newblock The international {HapMap} project.
\newblock \emph{Nature}, 426:\penalty0 789--796, 2003.

\bibitem[Wang et~al.(2005)Wang, Zhang, and Sheng]{WaZS05}
Rui-Sheng Wang, Xiang-Sun Zhang, and Li~Sheng.
\newblock Haplotype inference by pure parsimony via genetic algorithm.
\newblock In Xiang-Sun Zhang, De-Gang Liu, and Ling-Yun Wu, editors,
  \emph{Operations Research and Its Applications: the Fifth International
  Symposium (ISORA'05), Tibet, China, August 8--13}, volume~5 of \emph{Lecture
  Notes in Operations Research}, pages 308--318. Beijing World Publishing
  Corporation, Beijing, People Republic of China, 2005.

\end{thebibliography}
\bibliographystyle{plainnat}

\end{document}